\RequirePackage{fix-cm}
\documentclass[smallcondensed]{svjour3}     % onecolumn (ditto)
\smartqed  % flush right qed marks, e.g. at end of proof
\usepackage{graphicx}
\usepackage{epstopdf}
\usepackage{amsmath}
\usepackage{array}
\usepackage{subfigure}
\usepackage{tabularx}
\usepackage{dsfont}
\usepackage{amssymb}
\usepackage{bm}
\usepackage{url}
\usepackage{natbib}
\usepackage{yfonts}
\usepackage[colorlinks,linkcolor=blue,anchorcolor=red,citecolor=blue,urlcolor=blue]{hyperref}
\usepackage{geometry}
\DeclareMathOperator{\diag}{diag}
\begin{document}

\title{The Matrix Hilbert Space and Its Application to Matrix Learning%\thanks{Grants or other notes
}


\author{Yunfei Ye
}


\institute{Yunfei Ye \at
              Department of Mathematical Sciences, Shanghai Jiao Tong University \\
              800 Dongchuan RD Shanghai, 200240 China \\
              \email{tianshapojun@sjtu.edu.cn}           %  \\
}

\date{Received: date / Accepted: date}

\maketitle

\begin{abstract}
Theoretical studies have proven that the Hilbert space has remarkable performance in many fields of applications. Frames in tensor product of Hilbert spaces were introduced to generalize the inner product to high-order tensors. However, these techniques require tensor decomposition which could lead to the loss of information and it is a NP-hard problem to determine the rank of tensors. Here, we present a new framework, namely matrix Hilbert space to perform a matrix inner product space when data observations are represented as matrices. We preserve the structure of initial data and multi-way correlation among them is captured in the process. In addition, we extend the reproducing kernel Hilbert space (RKHS) to reproducing kernel matrix Hilbert space (RKMHS) and propose an equivalent condition of the space uses of the certain kernel function. A new family of kernels is introduced in our framework to apply the classifier of Support Tensor Machine(STM) and comparative experiments are performed on a number of real-world datasets to support our contributions.
\end{abstract}

\section{Introduction}
The Hilbert space was named after David Hilbert for his fundamental work to generalize the concept of Euclidean space to an infinite dimensional one in the field of functional analysis. With other related works, researchers have made great contribution in the development of quantum mechanics \citep{birkhoff1936logic,sakurai1995modern,ballentine2014quantum}, partial differential equations \citep{crandall1992user,gustafson2012introduction,gilbarg2015elliptic}, Fourier analysis \citep{stein2016introduction}, spectral theory \citep{birman2012spectral}, etc. Many methodologies have been proposed in the literature, but techniques are mainly studied based on infinite vector spaces. Since it is more natural to represent real-world data as high-order tensors, tensor product has become useful in approximating such variables. Khosravi and Asgari \citep{asgari2003frames} introduced frames in tensor product of Hilbert spaces. It is an extension of tensor product to construct a new Hilbert space of higher order tensors with several existing Hilbert spaces. Meanwhile, bases and frames in Hilbert $C^*$-modules with a $C^*$-algebra were investigated \citep{lance1995hilbert}. Tensor product of frames for Hilbert modules produce frames for a new Hilbert module \citep{khosravi2007frames}.

Based on this framework, a reproducing kernel Hilbert space (RKHS) of functions was proposed and proven essential in a number of applications, such as signal processing and detection, as well as statistical learning theory. The reproducing kernel was systematically developed in the early 1950s by Nachman Aronszajn \citep{aronszajn1950theory} and Stefan Bergman. The notion of kernels in Hilbert spaces wasn't brought to the field of machine learning until 20th century \citep{wahba1990spline,scholkopf1998nonlinear,vapnik1998statistical,boser1992training}. The kernel methods expand theories and algorithms well developed for the linear cases to nonlinear methods to detect the kind of dependencies that allow successful prediction of properties of interest \citep{hofmann2008kernel}.

Most of standard kernels use tensor decomposition to reveal the underlying structure of tensor data \citep{signoretto2011kernel,he2014dusk}. Existing tensor-based techniques consist of seeking representative low-dimensional subspaces or sum of rank-1 factor. However, information could be lost in this procedure and it is a NP-hard problem to determine the rank of tensors. Another matrix-based approach \citep{gao2015multiple} reformulates the Support Tensor Machine (STM) classifier where a matrix representation \citep{Gao2012Kernel} was applied in the construction of kernel function. Its improvement in the performance of classification problems attributes to the matrix kernel function which describes the inner product. Inspired by the above work, we study a new framework in this paper, namely matrix Hilbert space to perform inner product when data observations are represented as matrices. We exploit matrix inner product to capture structural information which could not be completely described by a simple scalar result. This includes in particular the case of Hilbert space where the properties of the scalar inner product are generalized. In addition, we systematically explain the matrix integral based on matrix polynomials \citep{sinap1994polynomial} upon our work. We begin by presenting the framework of matrix inner product space and extending it to the concept of matrix Hilbert space. Second, we develop tools extending to our matrix Hilbert space the concept of reproducing kernel matrix Hilbert space (RKMHS). 
To this end we derive an algorithm through alternating projection procedure to play our kernel trick.

The remainder of the article is structured as follows. In Sect.~\ref{sec:2} we present the framework of matrix Hilbert spaces combined with basic inequalities and properties. In Sect.~\ref{sec:3} we present definitions of reproducing kernel and corresponding reproducing kernel matrix Hilbert space. In Sect.~\ref{sec:4} a new family of kernels is described based on the framework of RKMHS and its performance on benchmark datasets. Finally, we present concluding remarks in Sect.~\ref{sec:5}.

\section{Matrix Inner product}
\label{sec:2}
In this section, we present the framework of matrix Hilbert space which extends the scalar inner product to a matrix form. One advantage is that it is a natural generalization of Hilbert space of vectors; this is the case especially when the dimension of the matrix inner product is one by one. Meanwhile, we reformulate some fundamental properties such as Cauchy-Schwarz inequality in our new space and obtain some good results.

In this study, scales are denoted by lowercase letters, e.g., s, vectors by boldface lowercase letters, e.g., \textbf{v}, matrices by boldface capital letters, e.g., \textbf{M} and general sets or spaces by gothic letters, e.g., $\mathcal{B}$. We start with some basic notations defined in the literature.

The Frobenius norm of a matrix $\textbf{X} \in \mathbb{R}^{m\times n}$ is defined by
\begin{equation*}
  \| \textbf{X} \|=\sqrt{\sum_{i_1=1}^{m} \sum_{i_2=1}^{n} x_{i_1  i_2}^2},
\end{equation*}
which is equal to the Euclidean norm of their vectorized representation.

The spectral norm of a matrix $\textbf{X} \in \mathbb{R}^{m\times n}$ is the square root of the largest eigenvalue of $\textbf{X}^\intercal \textbf{X}$:
\begin{equation*}
  \| \textbf{X} \|_{2}=\sqrt{\lambda_{\max}(\textbf{X}^\intercal \textbf{X})},
\end{equation*}
which is a natural norm induced by $l_2$ norm.

The inner product of two same-sized matrices $\textbf{X},\textbf{Y} \in \mathbb{R}^{m\times n}$ is defined as the sum of products of their entries, i.e.,
\begin{equation*}
  \langle \textbf{X},\textbf{Y} \rangle=\sum_{i_1=1}^{m} \sum_{i_2=1}^{n} x_{i_1 i_2}y_{i_1 i_2}.
\end{equation*}

Now we present our framework of matrix inner product as follows.

\begin{definition}[Matrix Inner Product]\label{ip}
Let $\mathcal{H}$ be a real linear space, the matrix inner product is a mapping $\langle \cdot, \cdot \rangle_{\mathcal{H}} : \mathcal{H} \times \mathcal{H} \rightarrow \mathbb{R}^{n \times n}$ satisfying the following properties, for all $\textbf{X}, \textbf{X}_1, \textbf{X}_2, \textbf{Y} \in \mathcal{H}$

(1) $\langle \textbf{Y},\textbf{X} \rangle_{\mathcal{H}} = \langle \textbf{X},\textbf{Y} \rangle_{\mathcal{H}} ^\intercal$

(2) $\langle \lambda\textbf{X}_1+\mu\textbf{X}_2,\textbf{Y} \rangle_{\mathcal{H}} = \lambda\langle \textbf{X}_1,\textbf{Y} \rangle_{\mathcal{H}}+\mu\langle \textbf{X}_2,\textbf{Y} \rangle_{\mathcal{H}}$

(3) $\langle \textbf{X},\textbf{X} \rangle_{\mathcal{H}}=\textbf{0}$ if and only if \textbf{X} is a zero element

(4) $\langle \textbf{X},\textbf{X} \rangle_{\mathcal{H}}$ is positive semidefinite. 
\end{definition}

\begin{remark}
Let $(\mathcal{H},\langle \cdot, \cdot \rangle_{\mathcal{H}})$ be a matrix inner product space. We assume that $\textbf{W} \in \mathbb{R}^{n \times n}$ is a symmetric matrix satisfying: $\langle \langle \textbf{X},\textbf{X} \rangle_{\mathcal{H}}, \frac{\textbf{W}}{\|\textbf{W}\|} \rangle \geq 0$ , where the case of equality holds precisely when $\textbf{X}$ is a zero element. The following function maps from matrix inner product to scalar inner product.
\begin{equation*}
\langle \cdot, \cdot \rangle_{\mathcal{H}} \stackrel{f}{\rightarrow} \langle \langle \cdot, \cdot \rangle_{\mathcal{H}},  \frac{\textbf{W}}{\|\textbf{W}\|} \rangle
\end{equation*}
\end{remark}

For example, $\langle \textbf{X},\textbf{Y} \rangle_{\mathcal{H}} =  \textbf{X}^\intercal \textbf{Y}$ is a simple case for the matrix inner product of matrix space $\mathcal{H}=\mathbb{R}^{m \times n}$. It simplifies to scalar inner product when $n=1$.

A matrical inner product on $\mathbb{R}^{n \times n}[x]$ defined by the matrix integral can be represented as

\begin{equation*}
\langle P(x),Q(x) \rangle_{\mathcal{H}}=\int_a^b P(x)^\intercal W(x) Q(x) dx
\end{equation*}
where $W(x)$ is a weight matrix function if , $p(x)$ and $Q(x)$ are polynomials in a real variable $x$ whose coefficients are $n \times n$ matrices \citep{sinap1994polynomial}. This can be applied in our framework as a special case where properties in Definition \ref{ip} are satisfied.

For convenience, all spaces that refer to $\mathcal{H}$ will be defined as $\mathbb{R}^{m \times n}$ without specification. The Cauchy-Schwarz inequality $|\langle \textbf{x},\textbf{y}\rangle|^2 \leq \langle \textbf{x},\textbf{x} \rangle \cdot \langle \textbf{y},\textbf{y} \rangle$ gives the upper bound of the inner product of two vectors. In the case of matrix inner product, we present the following inequality.  

\begin{theorem}\label{CST}
Let $(\mathcal{H},\langle \cdot, \cdot \rangle_{\mathcal{H}})$ be a matrix inner product space then
\begin{equation}
(\|\langle \emph{\textbf{X},\textbf{Y}} \rangle_{\mathcal{H}}\|_2)^2 \leq \|\langle \emph{\textbf{X},\textbf{X}} \rangle_{\mathcal{H}}\|_2 \cdot \|\langle \emph{\textbf{Y},\textbf{Y}} \rangle_{\mathcal{H}}\|_2.
\end{equation}
\end{theorem}

\begin{proof}
In the first step, it is natural to rewrite both sides of the inequality as 
\begin{equation}\label{rew}
\begin{split}
& \|\langle \textbf{X},\textbf{Y} \rangle_{\mathcal{H}}\|_2=\max_{\|\textbf{w}\|=1}|\textbf{w}^\intercal  \langle \textbf{X},\textbf{Y} \rangle_{\mathcal{H}} \textbf{w}| \\
& \|\langle \textbf{X},\textbf{X} \rangle_{\mathcal{H}}\|_2=\max_{\|\textbf{w}\|=1}\textbf{w}^\intercal  \langle \textbf{X},\textbf{X} \rangle_{\mathcal{H}} \textbf{w} \\
& \|\langle \textbf{Y},\textbf{Y} \rangle_{\mathcal{H}}\|_2=\max_{\|\textbf{w}\|=1}\textbf{w}^\intercal  \langle \textbf{Y},\textbf{Y} \rangle_{\mathcal{H}} \textbf{w},
\end{split}
\end{equation}
since $\langle \textbf{X},\textbf{X} \rangle_{\mathcal{H}}$ and $\langle \textbf{Y},\textbf{Y} \rangle_{\mathcal{H}}$ are positive semidefinite and $\textbf{w} \in \mathbb{R}^n$. From Definition \ref{ip}, we know that for all $\textbf{w} \in \mathbb{R}^n$ and $\lambda \in \mathbb{R}$, 

\begin{equation}\label{pbe}
\begin{split}
&\textbf{w}^\intercal \langle \textbf{X}-\lambda \textbf{Y},\textbf{X}-\lambda \textbf{Y} \rangle_{\mathcal{H}} \textbf{w} \geq 0 \\
\Rightarrow &\textbf{w}^\intercal  \langle \textbf{Y},\textbf{Y} \rangle_{\mathcal{H}} \textbf{w} \lambda^2-2 \textbf{w}^\intercal  \langle \textbf{X},\textbf{Y} \rangle_{\mathcal{H}} \textbf{w} \lambda + \textbf{w}^\intercal  \langle \textbf{X},\textbf{X} \rangle_{\mathcal{H}} \textbf{w} \geq 0 \\
\Rightarrow &  (\textbf{w}^\intercal  \langle \textbf{X},\textbf{Y} \rangle_{\mathcal{H}} \textbf{w} )^2 \leq \textbf{w}^\intercal  \langle \textbf{X},\textbf{X} \rangle_{\mathcal{H}} \textbf{w} \cdot \textbf{w}^\intercal  \langle \textbf{Y},\textbf{Y} \rangle_{\mathcal{H}} \textbf{w}.
\end{split}
\end{equation}

Suppose $\textbf{w}_p$ is the eigenvector corresponding to the eigenvalue $\lambda_p$ such that $|\lambda_p|=\max_{1 \leq i \leq n} |\lambda_i|$ where $\{\lambda_i\}_{i=1}^n$ are the eigenvalues of $\langle \textbf{X},\textbf{Y} \rangle_{\mathcal{H}}$. By directly combining Equation (\ref{rew}) and (\ref{pbe}), we have

\begin{equation}
\begin{split}
(\|\langle \textbf{X},\textbf{Y} \rangle_{\mathcal{H}}\|_2)^2 &=(\textbf{w}_p^\intercal  \langle \textbf{X},\textbf{Y} \rangle_{\mathcal{H}} \textbf{w}_p )^2 \\
& \leq \textbf{w}_p^\intercal  \langle \textbf{X},\textbf{X} \rangle_{\mathcal{H}} \textbf{w}_p \cdot \textbf{w}_p^\intercal  \langle \textbf{Y},\textbf{Y} \rangle_{\mathcal{H}} \textbf{w}_p \leq  \|\langle \textbf{X},\textbf{X} \rangle_{\mathcal{H}}\|_2 \cdot \|\langle \textbf{Y},\textbf{Y} \rangle_{\mathcal{H}}\|_2
\end{split}
\end{equation}
which concludes our proof.
\end{proof}
\qed

The difference between an inner product space and a Hilbert space is the assumption of completeness. The following definitions present a clear vision of completeness in our matrix space.

\begin{definition}[Convergence and Limit]\label{cl}
Let $(\mathcal{H},\langle \cdot, \cdot \rangle_{\mathcal{H}})$ be a matrix inner product space. We say that $\{\textbf{X}_k \in \mathcal{H}\}_{k=1}^{\infty}$ converges to $\textbf{X}$, written as $\lim_{k \rightarrow \infty} \textbf{X}_k = \textbf{X}$, if and only if 
\begin{equation*}
\lim_{k \rightarrow \infty} [\textbf{X}_k-\textbf{X}]_{i j}=0.
\end{equation*}
for all $i \in [1,m], j \in [1,n]$. $\textbf{X}_{ij}$ is the $(i,j)$ entry of \textbf{X}.
\end{definition}

\begin{definition}[Cauchy Sequence]
A sequence $\{\textbf{X}_i\}_{i=1}^{\infty}$ of a matrix inner product space $(\mathcal{H},\langle \cdot, \cdot \rangle_{\mathcal{H}})$ is called a Cauchy sequence, if for every real number $\epsilon>0$, there is a positive integer $N$ such that for all $p,q > N, i \in [1,m], j \in [1,n]$, $|[\textbf{X}_p-\textbf{X}_q]_{ij}|< \epsilon $.
\end{definition}

\begin{definition}[Complete Space]
A matrix inner product space $(\mathcal{H},\langle \cdot, \cdot \rangle_{\mathcal{H}})$ is called complete if every Cauchy sequence in $\mathcal{H}$ converges in $\mathcal{H}$.
\end{definition}

\begin{definition}[Matrix Hilbert Space]
A matrix Hilbert space is a complete matrix inner product space.
\end{definition}

The reason for the introduction of the matrix inner product in our matrix Hilbert space is that it contains more structural information in terms of initial data. We can capture the multi-way correlation in such framework that can hardly be expressed in scalar inner product. 

Let us now present some properties of the dual space which will help us in deriving some conclusion of mapping for specific cases of interest. 

We call the subset $\{\textbf{A}_i\}_{i=1}^{p}$ an orthogonal basis of matrix Hilbert space $\mathcal{H}$ if it satisfies the following properties:

(1) $\{\textbf{A}_i\}_{i=1}^{p}$ is linearly independent

(2) for all $\textbf{X} \in \mathcal{H}$, it can be decomposed as $\textbf{X}=\sum_i \lambda_i \textbf{A}_i$

(3) for all $i \neq j$, $\langle \textbf{A}_i,\textbf{A}_j \rangle_{\mathcal{H}}=\textbf{0}$.

Unfortunately, not every matrix Hilbert space has an orthogonal basis. For example, note that if $\mathcal{H}=\mathbb{R}^{m \times n}$ and $\langle \textbf{X},\textbf{Y} \rangle_{\mathcal{H}} =  \textbf{X}^\intercal \textbf{Y}$ where $n \geq 2$, then by the definition we have $p \geq mn$ to span the whole space. It follows that if the non-zero columns of $\{\textbf{A}_i\}_{i=1}^p$(each $\textbf{A}_i$ has at least one non-zero column to keep linearly independence) are orthogonal, then $p \leq m$ which is impossible.

\begin{definition}[Dual Space]
Let $(\mathcal{H},\langle \cdot, \cdot \rangle_{\mathcal{H}})$ be a matrix Hilbert space with an orthogonal basis $\{\textbf{A}_i\}_{i=1}^{p}$. Its dual space $\mathcal{H}^*$ contains linear mappings $f^*: \mathcal{H} \rightarrow \mathbb{R}^{n \times n}$ that for all $f^* \in \mathcal{H}^*$, $f^*(\textbf{A}_i)=\alpha_i \langle \textbf{A}_i,\textbf{A}_i \rangle_{\mathcal{H}}$ for all $i \in [1,p]$.
\end{definition}

It is indispensable to require a much stronger condition in the definition of dual space. One reason to explain this is that we aim to establish a certain connection between matrix Hilbert space and its dual space which will be discussed in Theorem \ref{Riesz}.  

Notice that the Riesz Representation Theorem constructs an isometrically isomorphic mapping between a Hilbert space and its dual space of real fields. Based on the framework of our matrix Hilbert space, we provide a weak representation theorem as follows.

\begin{theorem}[Weak Riesz Representation Theorem]\label{Riesz}
Let $(\mathcal{H},\langle \cdot, \cdot \rangle_{\mathcal{H}})$ be a matrix Hilbert space with an orthogonal basis $\{\emph{\textbf{A}}_i\}_{i=1}^{p}$ and $\mathcal{H}^*$ its dual space of $\mathcal{H}$. The mapping
\begin{equation}
\emph{\textbf{X}} \in \mathcal{H} \stackrel{f_{\emph{\textbf{X}}}}{\rightarrow} \langle \cdot,\emph{\textbf{X}} \rangle_{\mathcal{H}} \in \mathcal{H}^*
\end{equation}
is a linear isomorphism (i.e., it is injective and surjective).
\end{theorem}
\begin{proof}
To show the existence of the mapping, we rewrite $\textbf{X}=\sum_i \lambda_i \textbf{A}_i$ by the definition of orthogonal basis. Thus, $f_{\textbf{X}}(\textbf{A}_i)=\langle \textbf{A}_i,\textbf{X} \rangle_{\mathcal{H}}=\lambda_i \langle \textbf{A}_i,\textbf{A}_i \rangle_{\mathcal{H}}$ for all $i \in [1,p]$ which implies $f_{\textbf{X}} \in \mathcal{H}^*$. 

The mapping is linear by the property of matrix inner product. To show that the mapping is injective, we suppose that for $\textbf{X},\textbf{Y} \in \mathcal{H}$, $f_{\textbf{X}}(\textbf{Z})=f_{\textbf{Y}}(\textbf{Z})$ for all $\textbf{Z} \in \mathcal{H}$. Moreover, $f_{\textbf{X}}(\textbf{X}-\textbf{Y})=f_{\textbf{Y}}(\textbf{X}-\textbf{Y})$, which implies
\begin{equation*}
\langle \textbf{X}-\textbf{Y},\textbf{X}-\textbf{Y} \rangle_{\mathcal{H}}=\textbf{0}.
\end{equation*}
So $\textbf{X}=\textbf{Y}$ from the axioms of matrix inner product.

To show that the mapping is surjective, let $f^* \in \mathcal{H}^*$ which we assume without loss of generality is non-zero. Otherwise, \textbf{X} is the zero element. For all $\textbf{Y}=\sum_i \beta_i \textbf{A}_i \in \mathcal{H}$ by the definition of dual space,

\begin{equation}
\begin{split}
f^*(\textbf{Y})&=f^*(\sum_i \beta_i \textbf{A}_i)=\sum_i \beta_i f^*(\textbf{A}_i) \\
&=\sum_i \alpha_i \beta_i \langle \textbf{A}_i,\textbf{A}_i \rangle_{\mathcal{H}}=\langle \sum_i \beta_i \textbf{A}_i, \sum_i \alpha_i \textbf{A}_i \rangle_{\mathcal{H}} \\
&= \langle \textbf{Y}, \sum_i \alpha_i \textbf{A}_i \rangle_{\mathcal{H}},
\end{split}
\end{equation}
which concludes our proof.

\end{proof}
\qed

We therefore see that the matrix Hilbert space and its dual space are closely related to each other. In the next section we would make use of this connection to describe a mapping which can be applied in learning algorithms.

\section{Reproducing Kernel Matrix Hilbert Space}
\label{sec:3}

The Moore-Aronszajn theorem \citep{aronszajn1950theory} made the connection between kernel functions and Hilbert spaces. In the 20th century, Boser et al \citep{boser1992training} were the first to use kernels to construct a nonlinear estimation algorithm in the filed of machine learning. Over the last decades, researchers applied the technique of the kernel trick to nonlinear analysis problems rather than explicitly compute the high-dimensional coordinates in feature space. In this section, we extend the methodology of reproducing kernel Hilbert space (RKHS) to our reproducing kernel matrix Hilbert space (RKMHS). We develop the relationship between reproducing kernel and matrix Hilbert space. We begin with an intuitive definition of RKMHS.

\begin{definition}[Reproducing Kernel Matrix Hilbert Space]
Suppose $\mathcal{H}$ is a matrix Hilbert space of functions on domain $\mathcal{X}$($\textbf{X} \in \mathcal{X}, f \in \mathcal{H}$ and $f(\textbf{X}) \in \mathbb{R}^{n \times n}$), for each $\textbf{Y} \in \mathcal{X}$ if $f(\textbf{Y}): \mathcal{H} \rightarrow \mathbb{R}^{n \times n}$ is in its dual space, then according to Theorem \ref{Riesz} there exists a function $K_{\textbf{Y}}$ of $\mathcal{H}$ with the property,
\begin{equation*}
f(\textbf{Y})=\langle f,K_{\textbf{Y}} \rangle_{\mathcal{H}}.
\end{equation*}
Since $K_{\textbf{X}}$ is itself a function in $\mathcal{H}$, we have that for each $\textbf{X} \in \mathcal{X}$
\begin{equation*}
K_{\textbf{X}}(\textbf{Y})=\langle K_{\textbf{X}},K_{\textbf{Y}} \rangle_{\mathcal{H}}
\end{equation*}
The reproducing kernel of $\mathcal{H}$ is a function $K: \mathcal{X} \times \mathcal{X} \rightarrow \mathbb{C}^{n \times n}$ defined by
\begin{equation*}
K(\textbf{X},\textbf{Y})=\langle K_{\textbf{X}},K_{\textbf{Y}} \rangle_{\mathcal{H}}.
\end{equation*}
We call such space $\mathcal{H}$ a RKMHS.
\end{definition}

We can easily derive that for all $\textbf{X}_i,\textbf{X}_j \in \mathcal{X}, \alpha_i, \alpha_j \in \mathbb{R}, m \in \mathbb{N}$
\begin{equation*}
\begin{split}
\sum_{i,j=1}^m \alpha_i \alpha_j K(\textbf{X}_i,\textbf{X}_j)&=\sum_{i,j=1}^m \alpha_i \alpha_j \langle K_{\textbf{X}_i},K_{\textbf{X}_j} \rangle_{\mathcal{H}} \\
&=\langle \sum\limits_{i=1}^m \alpha_i K_{\textbf{X}_i},\sum\limits_{j=1}^m \alpha_j K_{\textbf{X}_j} \rangle_{\mathcal{H}} 
\end{split}
\end{equation*}
is positive semidefinite. 

And if there exist $\textbf{X}_i,\textbf{X}_j \in \mathcal{X}, \alpha_i, \alpha_j \in \mathbb{R}, m \in \mathbb{N}$,
\begin{equation*}
\sum_{i,j=1}^m \alpha_i \alpha_j K(\textbf{X}_i,\textbf{X}_j)=\textbf{0},
\end{equation*}
then $\sum_{i=1}^m \alpha_i K_{\textbf{X}_i}$ is zero.

More generally, we use mapping to obtain the definition of kernel.

\begin{definition}[Kernel]
A function
\begin{equation*}
K: \mathcal{X} \times \mathcal{X} \rightarrow \mathbb{R}^{n \times n}, (\textbf{X},\textbf{X}') \mapsto K(\textbf{X},\textbf{X}')
\end{equation*}
satisfying for all $\textbf{X},\textbf{X}' \in \mathcal{X}$
\begin{equation}\label{kk}
K(\textbf{X},\textbf{X}')=\langle \Phi(\textbf{X}), \Phi(\textbf{X}') \rangle_{\mathcal{H}}
\end{equation}
is called a kernel where the feature map $\Phi$ maps into some matrix Hilbert space $\mathcal{H}$, or the feature space.
\end{definition}

Now we define a positive semidefinite kernel as follows.

\begin{definition}[Positive Semidefinite Kernel]\label{psdkw}
A function $K: \mathcal{X} \times \mathcal{X} \rightarrow \mathbb{C}^{n \times n}$ satisfying
\begin{equation}\label{p1}
K(\textbf{X}_i,\textbf{X}_j)=K(\textbf{X}_j,\textbf{X}_i)^{\intercal}
\end{equation}
\begin{equation}\label{p2}
\sum_{i,j=1}^m \alpha_i \alpha_j K(\textbf{X}_i,\textbf{X}_j) \ is \ positive \ semidefinite
\end{equation}
for all $\textbf{X}_i,\textbf{X}_j \in \mathcal{X}, \alpha_i, \alpha_j \in \mathbb{R}, m \in \mathbb{N}$ is called a positive semidefinite kernel.
\end{definition}

Now we show that the class of kernels that can be written in the form (\ref{kk}) coincides with the class of positive semidefinite kernels.

\begin{theorem}\label{Moo}
To every positive semidefinite function K on $\mathcal{X} \times \mathcal{X}$, there corresponds a RKMHS $\mathcal{H}_K$ of real-valued functions on $\mathcal{X}$ and vice versa.
\end{theorem}
\begin{proof}
Suppose $\mathcal{H}_{K}$ is a RKMHS, the properties of (\ref{p1}) and (\ref{p2}) follow from the definition of matrix Hilbert space for reproducing kernel $K(\textbf{X},\textbf{Y})=\langle K_{\textbf{X}},K_{\textbf{Y}} \rangle_{\mathcal{H}}$.

We then suggest how to construct $\mathcal{H}_K$ given K. Let $K_{\textbf{X}}(\textbf{Y})=K(\textbf{X},\textbf{Y})$ denotes the function of \textbf{Y} obtained by fixing \textbf{X}. Consider $\mathcal{H}$ be all linear combinations in $span\{K_{\textbf{X}} | \textbf{X} \in \mathcal{X}\}$
\begin{equation}\label{cork}
f(\cdot)=\sum_{i=1}^m \alpha_i K_{\textbf{X}_i}(\cdot).
\end{equation}
Here, $\textbf{X}_i \in \mathcal{X}, \alpha_i \in \mathbb{R}, m \in \mathbb{N}$ are arbitrary.

Next, we define the matrix inner product between \emph{f} and another function $g(\cdot)=\sum_{j=1}^{m'} \beta_j K_{\textbf{X}_j'}(\cdot)$ as
\begin{equation}
\langle f, g \rangle_{\mathcal{H}}=\sum_{i=1}^m \sum_{j=1}^{m'} \alpha_i \beta_j K(\textbf{X}_i, \textbf{X}_j').
\end{equation}

Note that $\langle f, g \rangle_{\mathcal{H}}=\sum_{i=1}^{m} \alpha_i \sum_{j=1}^{m'} \beta_j K(\textbf{X}_i, \textbf{X}_j')$ which shows that $\langle \cdot, \cdot \rangle$ is linear. $\langle f, g \rangle_{\mathcal{H}}=\langle g, f \rangle_{\mathcal{H}}^{\intercal}$, as $K(\textbf{X}_i, \textbf{X}_j')=K(\textbf{X}_j', \textbf{X}_i)^{\intercal}$. Moreover, for any function $f$, written as (\ref{cork}), we have
\begin{equation}\label{psd}
\langle f,f \rangle_{\mathcal{H}}=\sum_{i,j=1}^m \alpha_i \alpha_j K(\textbf{X}_i,\textbf{X}_j) \ is \ positive \ semidefinite.
\end{equation}

To prove that $K_{\textbf{X}}$ is the reproducing kernel of $\mathcal{H}$, we can easily derive that
\begin{equation*}
\langle f, K_{\textbf{Y}} \rangle_{\mathcal{H}}=\sum_{i=1}^m \alpha_i K(\textbf{X}_i,\textbf{Y})=f(\textbf{Y}), \ \langle K_{\textbf{X}},K_{\textbf{Y}} \rangle_{\mathcal{H}}=K(\textbf{X},\textbf{Y}),
\end{equation*}
which proves the reproducing property.

For the last step in proving that $\mathcal{H}$ is a matrix inner product space, due to (\ref{psd}) and Theorem \ref{CST}, we have
\begin{equation*}
(\|f(\textbf{Y})\|_2)^2=(\|\langle f, K_{\textbf{Y}} \rangle_{\mathcal{H}}\|_2)^2 \leq \|\langle f,f \rangle_{\mathcal{H}} \|_2  \| \langle K_{\textbf{Y}},K_{\textbf{Y}} \rangle_{\mathcal{H}}\|_2
\end{equation*}
for all $\textbf{Y} \in \mathcal{X}$. By this inequality, $\langle f,f \rangle_{\mathcal{H}}=\textbf{0}$ implies $\|f(\textbf{Y})\|_2=0$ and $f(\textbf{Y})=\textbf{0}$, which is the last property that was left to prove in order to establish that $\langle \cdot,\cdot \rangle_{\mathcal{H}}$ is a matrix inner product. Thus, the completion of $\mathcal{H}$ which can be denoted by $\mathcal{H}_K$ is a RKMHS. 
\end{proof}
\qed

In such case we naturally define a mapping $\Phi: \textbf{X} \rightarrow K_{\textbf{X}}(\cdot)$ that enable us to establish a matrix Hilbert space through positive semidefinite kernel. Note that we do not include the uniqueness of corresponding space in the above Theorem, this is mainly because we can not span the whole space by a closed subset and its orthogonal complement induced by Hilbert space.

Now we introduce some closure properties of the set of positive semidefinite kernels.

\begin{proposition}
Below, $K_1,\ldots, K_m$ are arbitrary positive semidefinite kernels on $\mathcal{X} \times \mathcal{X}$ , where $\mathcal{X}$ is a nonempty set:

\emph{(1)} For all $\lambda_1, \lambda_2 \geq 0$, $\lambda_1K_1+ \lambda_2 K_2$ is positive semidefinite.

\emph{(2)} If $K(\emph{\textbf{X}},\emph{\textbf{X}}')=\lim_{m \rightarrow \infty} K_m(\emph{\textbf{X}},\emph{\textbf{X}}')$ for all $\emph{\textbf{X}},\emph{\textbf{X}}' \in \mathcal{X}$, then $K$ is  positive semidefinite.
\end{proposition}

The proof is trivial. Some possible choices of $K$ include
\begin{equation*}
\begin{split}
&\rm{Linear \ kernel:} \qquad K(\textbf{X},\textbf{Y})=\textbf{X}^\intercal \textbf{Y}, \\
&\rm{Polynomial \ kernel:} \quad K(\textbf{X},\textbf{Y})= (\textbf{X}^\intercal \textbf{Y}+\alpha \textbf{I}_{n \times n})^{\circ \beta}\\
&\rm{Gaussian \ kernel:} \quad  K(\textbf{X},\textbf{Y})=[\exp(-\gamma \|\textbf{X}(:,i)-\textbf{Y}(:,j)\|^2)]_{n \times n}
\end{split}
\end{equation*}
where $\alpha\geq 0, \beta \in \mathbb{N}, \gamma >0, \textbf{X}, \textbf{Y} \in \mathcal{H}=\mathbb{R}^{m \times n}$. $\textbf{X}(:,i)$ is the $i$-th column of \textbf{X} and $\circ$ is the Hadamard product \citep{horn1990hadamard}.

The above kernel design is by no means complete. Any construction satisfying the properties of positive semidefinite can be applied in practice.

\section{Experiments}
\label{sec:4}
In this section, we demonstrate the use of our matrix Hilbert Space in practice. We point out a connection to the Support Tensor Machine (STM) in the field of machine learning. We propose a family of matrix kernels to estimate the similarity of matrix data. Note that by doing so we essentially recover the matrix-based approach \citep{gao2015multiple} and construct the kernel by the methodology of our space. We use the real world data to evaluate the performance of different kernels (DuSK \citep{he2014dusk}, factor kernel \citep{signoretto2011kernel}, linear, Gaussian-RBF, MRMLKSVM \citep{Gao2014NLS} and ours) on SVM or STM classifier, since they have been proven successful in various applications.

All experiments were conducted on a computer with Intel(R) Core(TM) i5 (3.30 GHZ) processor with 16.0 GB RAM memory. The algorithms were implemented in Matlab.

\subsection{Algorithms}
Given a set of samples $\{(y_i,\textbf{X}_i)\}_{i=1}^N$ for binary classification problem, where $\textbf{X}_i \in \mathbb{R}^{m \times n}$ are the input matrix data and $y_i \in \{-1,+1\}$ are the corresponding class labels. Linear STM aims to find the separating hyperplane $f(\textbf{X})=\langle \textbf{W},\textbf{X}\rangle+b=0$, it can be evaluated by considering the following question \citep{hao2013linear}:
\begin{equation}\label{bp}
\begin{split}
  &\min_{\textbf{W},b,\bm{\xi}} \ \frac{1}{2}\|\textbf{W}\|_{F}^2 + C \sum_{i=1}^N \xi_i \\
  &s.t. \ y_i(\langle \textbf{W},\textbf{X}_i\rangle+b)\geq 1-\xi_i, \ 1 \leq i \leq N \\
  &\quad \ \ \bm{\xi} \geq 0,
\end{split}
\end{equation}
where $\textbf{W} \in \mathbb{R}^{m \times n}$, $\bm{\xi}=[\xi_1, \cdots, \xi_N]^T$ is the vector of all slack variables of training examples. By using the technique of the singular value decomposition (SVD) for matrix \textbf{W}, we have 

\begin{equation}
\textbf{W}=\overline{\textbf{U}}
\left[ \begin{array}{cc}
\Sigma & 0 \\
0 & 0
\end{array} \right]
\overline{\textbf{V}}^{\intercal},
\end{equation}
where $\overline{\textbf{U}}=[\overline{\textbf{u}}_1,\cdots,\overline{\textbf{u}}_m] \in \mathbb{R}^{m \times m}$, $\overline{\textbf{V}}=[\overline{\textbf{v}}_1,\cdots,\overline{\textbf{v}}_n] \in \mathbb{R}^{n \times n}$, $\Sigma=\diag(\sigma_1^2,\cdots,\sigma_r^2)$ and r is the rank of $\textbf{W}$. Let $\textbf{u}_k=\sigma \overline{\textbf{u}}_k$ and $\textbf{v}_k=\sigma \overline{\textbf{v}}_k$, we will reformulate $\textbf{W}=\sum\limits_{k=1}^r \textbf{u}_k \textbf{v}_k^{\intercal}$.

Substituting $\textbf{W}$ for that in problem (\ref{bp}) and mapping \textbf{X} to the feature space $\Phi: \textbf{X} \rightarrow \Phi(\textbf{X})$, we can reformulate the optimization problem as \citep{gao2015multiple}:
\begin{equation}\label{rp}
\begin{split}
  &\min_{\textbf{u},\textbf{v},b,\bm{\xi}} \ \frac{1}{2}\sum_{k=1}^r (\textbf{u}_k^{\intercal}\textbf{u}_k)(\textbf{v}_k^{\intercal}\textbf{v}_k) + C \sum_{i=1}^N \xi_i \\
  &s.t. \ y_i(\sum_{k=1}^r \textbf{u}_k^{\intercal} \Phi(\textbf{X}_i) \textbf{v}_k+b)\geq 1-\xi_i, \ 1 \leq i \leq N \\
  &\quad \ \ \bm{\xi} \geq 0.
\end{split}
\end{equation}

We minimize the objective function iteratively. In each iteration, we first fix $\{\textbf{v}_k \in \mathbb{R}^n\}_{k=1}^r$ and derive $\{\textbf{u}_k \in \mathbb{R}^m\}_{k=1}^r$ by solving the Lagrangian dual problem of the primal optimization problem. Then, we fix $\{\textbf{u}_k \in \mathbb{R}^{m}\}_{k=1}^r$ and do the similar process.

For any given nonzero vectors $\{\textbf{v}_k \in \mathbb{R}^{n}\}_{k=1}^r$, the Lagrangian function for this problem is
\begin{equation}
\begin{split}
L=&\frac{1}{2}\sum_{k=1}^r (\textbf{u}_k^{\intercal}\textbf{u}_k)(\textbf{v}_k^{\intercal}\textbf{v}_k) + C \sum_{i=1}^N \xi_i \\
&-\sum_{i=1}^N \alpha_i (y_i[\sum_{k=1}^r \textbf{u}_k^{\intercal} \Phi(\textbf{X}_i) \textbf{v}_k+b]-1+\xi_i)-\sum_{i=1}^N \gamma_i \xi_i,
\end{split}
\end{equation}
with Lagrangian multipliers $\alpha_i \geq 0$, $\gamma_i \geq 0$ for $1 \leq i \leq N$.

The derivative of $L$ with respect to $\textbf{u}_k$, $\xi_i$ and b give
\begin{equation}\label{ie}
\begin{split}
&\frac{\partial L}{\partial \textbf{u}_k}=0 \Rightarrow \textbf{u}_k=\sum\limits_{i=1}^N \frac{1}{\textbf{v}_k^T\textbf{v}_k}\alpha_i y_i \Phi(\textbf{X}_i) \textbf{v}_k \\
&\frac{\partial L}{\partial \xi_i}=0 \Rightarrow \sum_{i=1}^N \alpha_i+\gamma_i=C \\
&\frac{\partial L}{\partial b}=0 \Rightarrow \sum_{i=1}^N \alpha_i y_i=0.
\end{split}
\end{equation}

Substituting $K(\textbf{X}_i,\textbf{X}_j)$ into $\Phi^{\intercal}(\textbf{X}_i)\Phi(\textbf{X}_j)$ and problem (\ref{rp}) can be simplified as a standard SVM expression
\begin{equation} \label{n1}
\begin{split}
  &\min_{\bm{\alpha}} \ \frac{1}{2}\sum_{i=1}^N \sum_{j=1}^N y_i y_j \alpha_i \alpha_j \sum_{k=1}^r \frac{1}{\textbf{v}_k^{\intercal}\textbf{v}_k}\textbf{v}_k^TK(\textbf{X}_i,\textbf{X}_j)\textbf{v}_k -\sum_{i=1}^N \alpha_i \\
  &s.t. \ \sum_{i=1}^N \alpha_i y_i=0, \\
  &\quad \ \ 0 \leq \bm{\alpha} \leq C.
\end{split}
\end{equation}

Once vector $\bm{\alpha}$ is obtained, $\{\textbf{u}_k\}_{k=1}^r$ can be calculated by Eq. (\ref{ie}) though we can not explicitly express $\Phi(\textbf{X}_i)$. For any given nonzero vectors $\{\textbf{u}_k \in \mathbb{R}^m\}_{k=1}^r $, do the same process. The equivalent formulation of problem (\ref{rp}) is
\begin{equation} \label{n2}
\begin{split}
  &\min_{\bm{\alpha}} \ \frac{1}{2}\sum_{i=1}^N \sum_{j=1}^N y_i y_j \beta_i \beta_j \sum_{k=1}^r \frac{1}{\textbf{u}_k^T\textbf{u}_k}\textbf{u}_k^{\intercal}\Phi(\textbf{X}_i)\Phi^{\intercal}(\textbf{X}_j)\textbf{u}_k -\sum_{i=1}^N \beta_i \\
  &s.t. \ \sum_{i=1}^N \beta_i y_i=0, \\
  &\quad \ \ 0 \leq \bm{\beta} \leq C.
\end{split}
\end{equation}

Similarly, we can simply derive that
\begin{equation}
\textbf{v}_k=\sum\limits_{i=1}^N \frac{1}{\textbf{u}_k^{\intercal}\textbf{u}_k}\beta_i y_i \Phi^{\intercal}(\textbf{X}_i) \textbf{u}_k.
\end{equation}

Notice that
\begin{equation}
\begin{split}
&\textbf{u}_k^{\intercal}\textbf{u}_k=\frac{1}{(\textbf{v}_k^{\intercal}\textbf{v}_k)^2}\sum\limits_{i,j=1}^N\alpha_i\alpha_j y_i y_j \textbf{v}_k^{\intercal} K(\textbf{X}_i, \textbf{X}_j)\textbf{v}_k,\\
&\textbf{u}_k^{\intercal} \Phi(\textbf{X}_i)=\sum\limits_{j=1}^N \frac{1}{\textbf{v}_k^{\intercal}\textbf{v}_k}\alpha_j y_j \textbf{v}_k^{\intercal} K(\textbf{X}_j, \textbf{X}_i),
\end{split}
\end{equation}
where $\{\textbf{u}_k\}_{k=1}^r$ are calculated by $\{\textbf{v}_k\}_{k=1}^r$ in the previous iteration.

We can iteratively deal with such process. If $\|\bm{\alpha}^{New}-\bm{\alpha}^{Old}\|<\varepsilon, \|\bm{\beta}^{New}-\bm{\beta}^{Old}\|<\varepsilon$ or the maximum number of iterations is achieved, stop iteration. 

Note that if $\Phi(\textbf{X}_i)=\textbf{X}_i$ and $r=1$ in the process above, the method becomes STM algorithm in matrix case.

We now derive one new family of kernel in the framework of matrix Hilbert space. Using the technique of Singular Value Decomposition (SVD), a matrix \textbf{X} $\in \mathbb{R}^{m \times n}$ can be decomposed in block-partitioned form as
\begin{equation}
\textbf{X}=[\textbf{U}_{\textbf{X}}, \widetilde{\textbf{U}}_{\textbf{X}}]
\left[ \begin{array}{cc} \textbf{S}_{\textbf{X}} & 0 \\ 0 & 0 \end{array}\right]
[\textbf{V}_{\textbf{X}}, \widetilde{\textbf{V}}_{\textbf{X}}]^{\intercal}.
\end{equation}
where $\textbf{U}_{\textbf{X}} \in \mathbb{R}^{m \times c}, \textbf{V}_{\textbf{X}} \in \mathbb{R}^{n \times c}$ and $c=\min\{m,n\}$. Let $\textbf{W}_{\textbf{X}}=[\textbf{U}_{\textbf{X}}^{\intercal},\textbf{V}_{\textbf{X}}^{\intercal}]^{\intercal}$. $\textbf{U}_{\textbf{X}}$ can be divided by columns as $[\textbf{U}_{\textbf{X},1},\ldots,\textbf{U}_{\textbf{X},c}]$, so as $\textbf{V}_{\textbf{X}}$ and $\textbf{W}_{\textbf{X}}$.

\begin{definition}[New matrix kernel]\label{nmk}
For $\textbf{X},\textbf{Y} \in \mathcal{H}$, let $\Phi : \mathbb{R}^{m \times n} \rightarrow \mathcal{H}$ be a mapping where $\Phi(\textbf{X})=[\Phi(\textbf{W}_{\textbf{X},1}),\ldots,\Phi(\textbf{W}_{\textbf{X},c})]$. We then define the kernel function $K$ as
\begin{equation}
\begin{split}
K(\textbf{X},\textbf{Y})&= \Phi^{\intercal}(\textbf{X})\Phi(\textbf{Y}) \\
&=[\Phi^{\intercal}(\textbf{W}_{\textbf{X},i})\Phi(\textbf{W}_{\textbf{Y},j})]_{c \times c},
\end{split}
\end{equation}
i.e., the $(i,j)$ entry of $K(\textbf{X},\textbf{Y})$ is $\Phi^{\intercal}(\textbf{W}_{\textbf{X},i})\Phi(\textbf{W}_{\textbf{Y},j})$ for $ i \in [1,c], j \in [1,c]$.
\end{definition}

\begin{remark}
SVD is introduced to build kernel function because the left-singular vectors and right-singular vectors are orthonormal bases of the spaces spanned by columns and rows of the data respectively. They contain compact and structural information of matrix objects.
\end{remark}

The following Theorem immediately yields that, in general, the above description gives a meaningful construction of matrix kernel function. 

\begin{theorem}
Suppose that $\Phi^{\intercal}(\textbf{x}_i)\Phi(\textbf{y}_j)=k(\textbf{x}_i,\textbf{y}_j)$ is any valid kernel of vector space (e.g. Gaussian RBF kernel, polynomial kernel) in Definition \ref{nmk}, then the kernel function $K$ is a positive semidefinite kernel.
\end{theorem}

The proof is trivial. According to Theorem \ref{Moo}, we just need to show that the kernel function satisfies  conditions in Definition \ref{psdkw} so we leave it to readers.

\subsection{Datasets}

Next, we compare the performance of our kernels with those introduced in the literature: DuSK, linear, Gaussian-RBF, factor kernel and MRMLKSVM on real data sets. We consider the following benchmark datasets to perform a series of comparative experiments on both binary and multiple classification problems. We use the MNIST database \citep{lecun1998gradient} of handwritten digits established by Yann LeCun, etc. from \url{http://yann.lecun.com/exdb/mnist/}, the Yale Face database \citep{belhumeur1997eigenfaces} from \url{http://vision.ucsd.edu/content/yale-face-database} and the FingerDB database from \url{http://bias.csr.unibo.it/fvc2000/databases.asp}. To better visualize the experimental data, we randomly choose a small subset for each database, as shown in Fig. \ref{Fig1}.

The MNIST database of handwritten digits has a training set of 60,000 examples, and a test set of 10,000 examples. Each image was centered in a $28 \times 28$ image by computing the center of mass of the pixels, and translating the image so as to position this point at the center of the $28 \times 28$ field. For each digit, we compare it with other digits to deal with the binary classification problem. We randomly choose 100 and 1000 examples respectively as the training set while 1000 arbitrary examples are used as test set. In addition, half of them are of one digit while the remaining are of other digits.

The Yale Face database contains 165 grayscale images in GIF format of 15 individuals with $243 \times 320$ pixels. There are 11 images per subject, one per different facial expression or configuration. In the experiment, we randomly pick up 6 images of each individual as the training set and other images remained for testing for multiple classification.

The FingerDB database contains 80 finger images of 10 individuals with $300 \times 300$ pixels, and each individual has 8 finger images. In this experiment, 4 randomly selected images of each individual were gathered as training set and the rest images retained as test set for the multiple classification. This random selection experiment was repeated 20 times for all data sets to obtain the average and standard deviation of performance measures.

\begin{figure}
\centering
\subfigure[]{
\includegraphics[width=84mm]{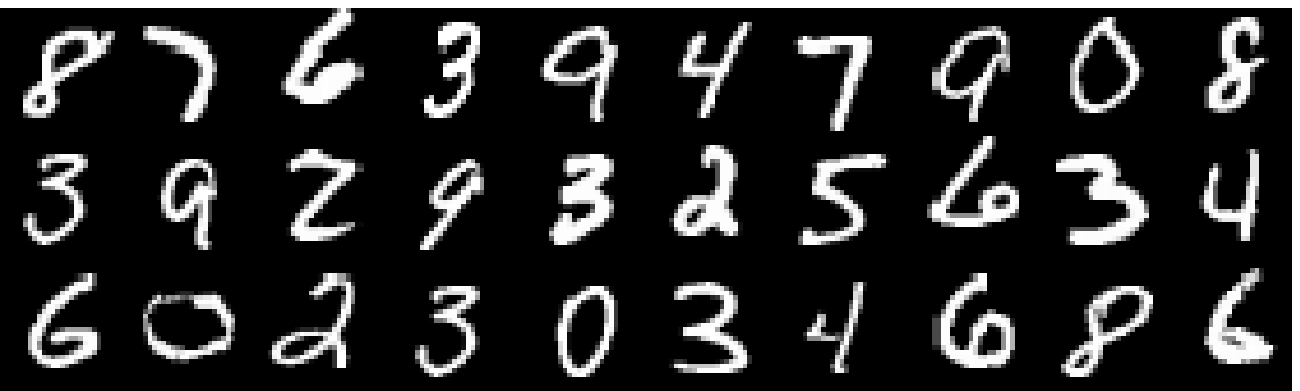}}
\subfigure[]{
\includegraphics[width=84mm]{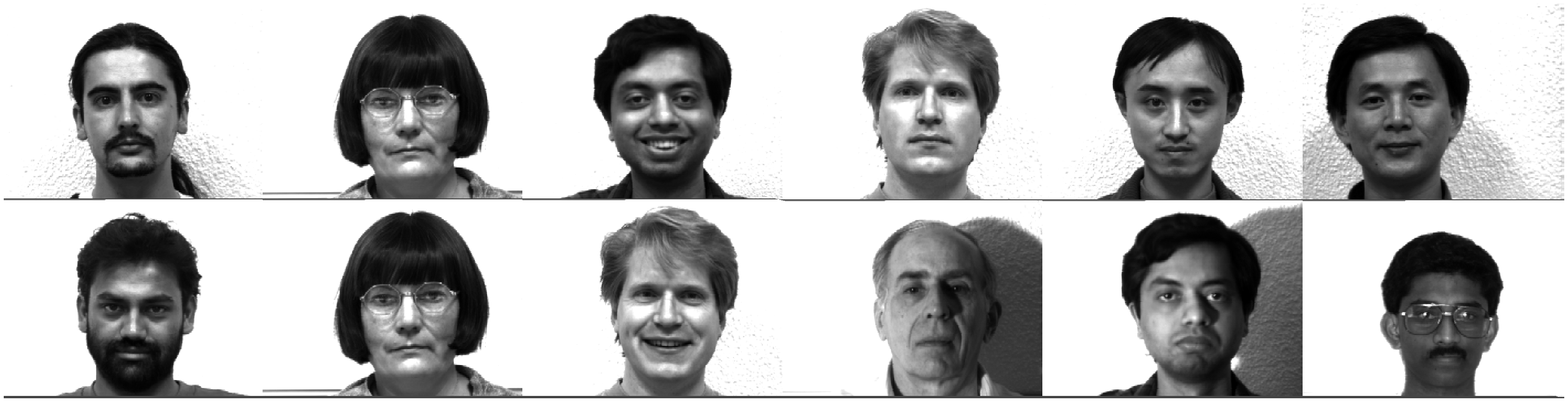}}
\subfigure[]{
\includegraphics[width=84mm]{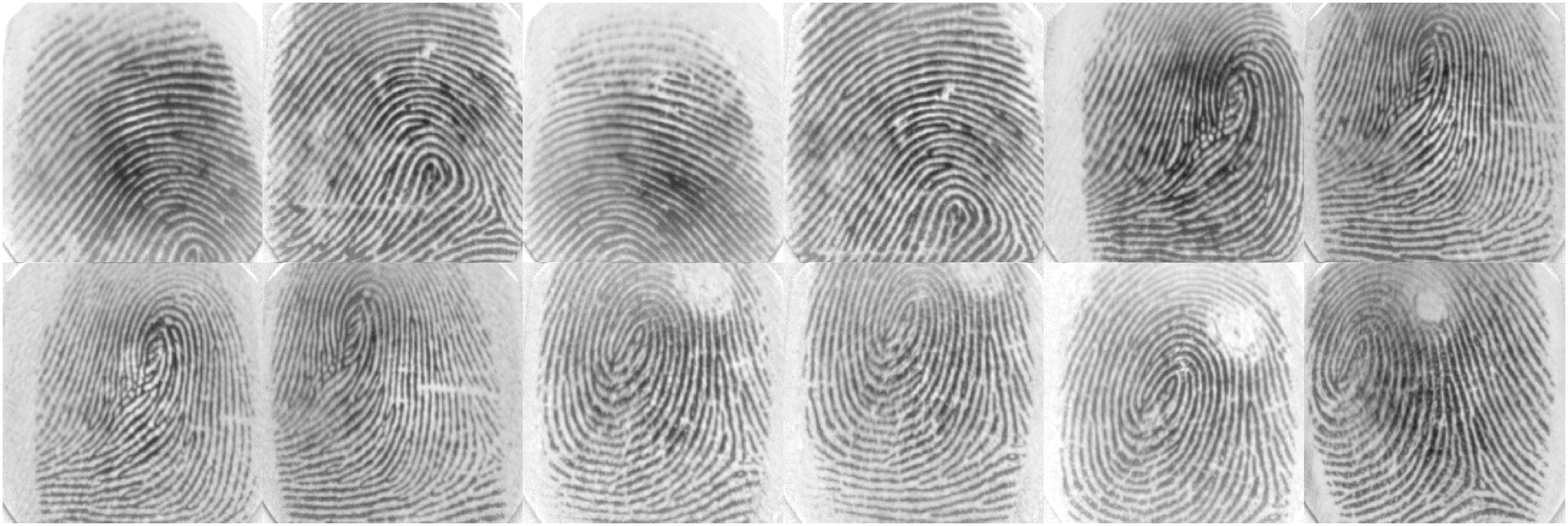}}
\caption{Examples for matrix datasets.  \textbf{a} MNIST samples. \textbf{b} Yale face samples. \textbf{c} FingerDB samples}
\label{Fig1}
\end{figure}

One-to-one classification method is introduced when it comes to multiple classification problem. All data are standardized into $[0,1]$ through linear transformation. The input matrices are converted into vectors when it comes to the SVM problems. All kernels select the optimal trade-off parameter $C \in \{10^{-2},10^{-1},\ldots,10^2\}$, kernel width parameter $\sigma \in \{10^{-4},10^{-3},\ldots,10^4\}$ and rank $r \in \{1,2,\cdots,10\}$. All the learning machines use the same training and test set. For the purpose of parameter selection, grid search is introduced in experiments. In MRMLKSVM and new matrix kernel, Gaussian RBF kernel $k(\textbf{x},\textbf{y})=\exp(-\sigma \|\textbf{x}-\textbf{y}\|^2)$ and polynomial kernel $k(\textbf{x},\textbf{y})=(\textbf{x}^T\textbf{y}+\sigma)^2$ are used respectively as the vector kernel functions. Gaussian RBF kernel is used in DuSK which denoted as $\rm DuSK_{RBF}$. In addition, linear and Gaussian-RBF kernel are introduced on SVM classifier which denoted as $\rm SVM_{linear}$ and $\rm SVM_{RBF}$ respectively.

To evaluate the performance of the different kernels, we introduce two performance measures. We report the accuracy which counts on the proportion of correct predictions, the $F_1$ score as the harmonic mean of precision and recall $F_1=2 \cdot \frac{Pre \times Rec}{Pre+Rec}$. Precision is the fraction of retrieved instances that are relevant, while recall is the fraction of relevant instances that are retrieved. In multiple classification problems, macro-averaged F-measure \citep{yang1999re} is adopted as the average of $F_1$ score for each category.

\subsection{Discussions}

\begin{figure}
\centering
\subfigure[100 train]{
\includegraphics[width=0.48\textwidth]{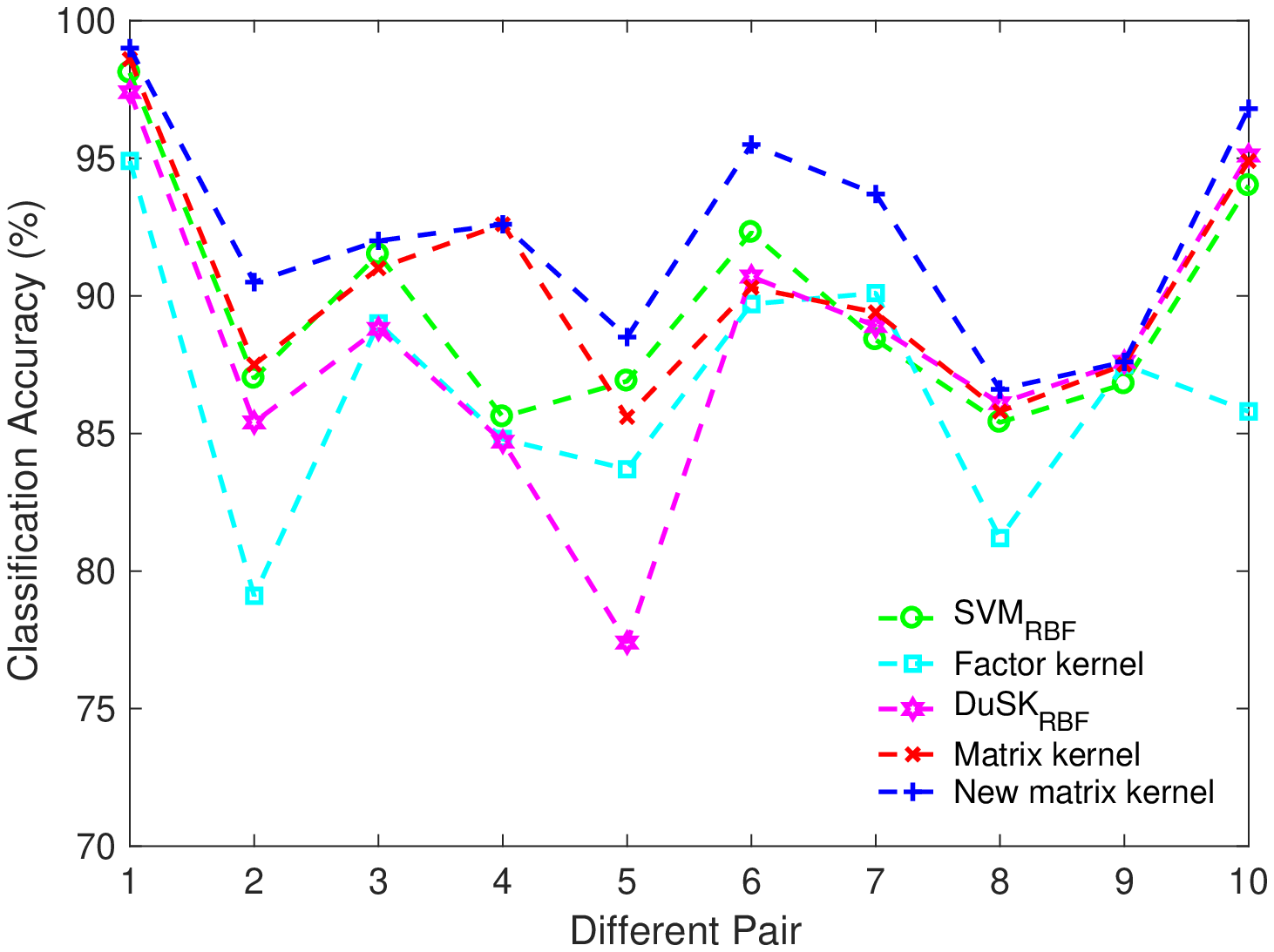}}
\subfigure[1000 train]{
\includegraphics[width=0.48\textwidth]{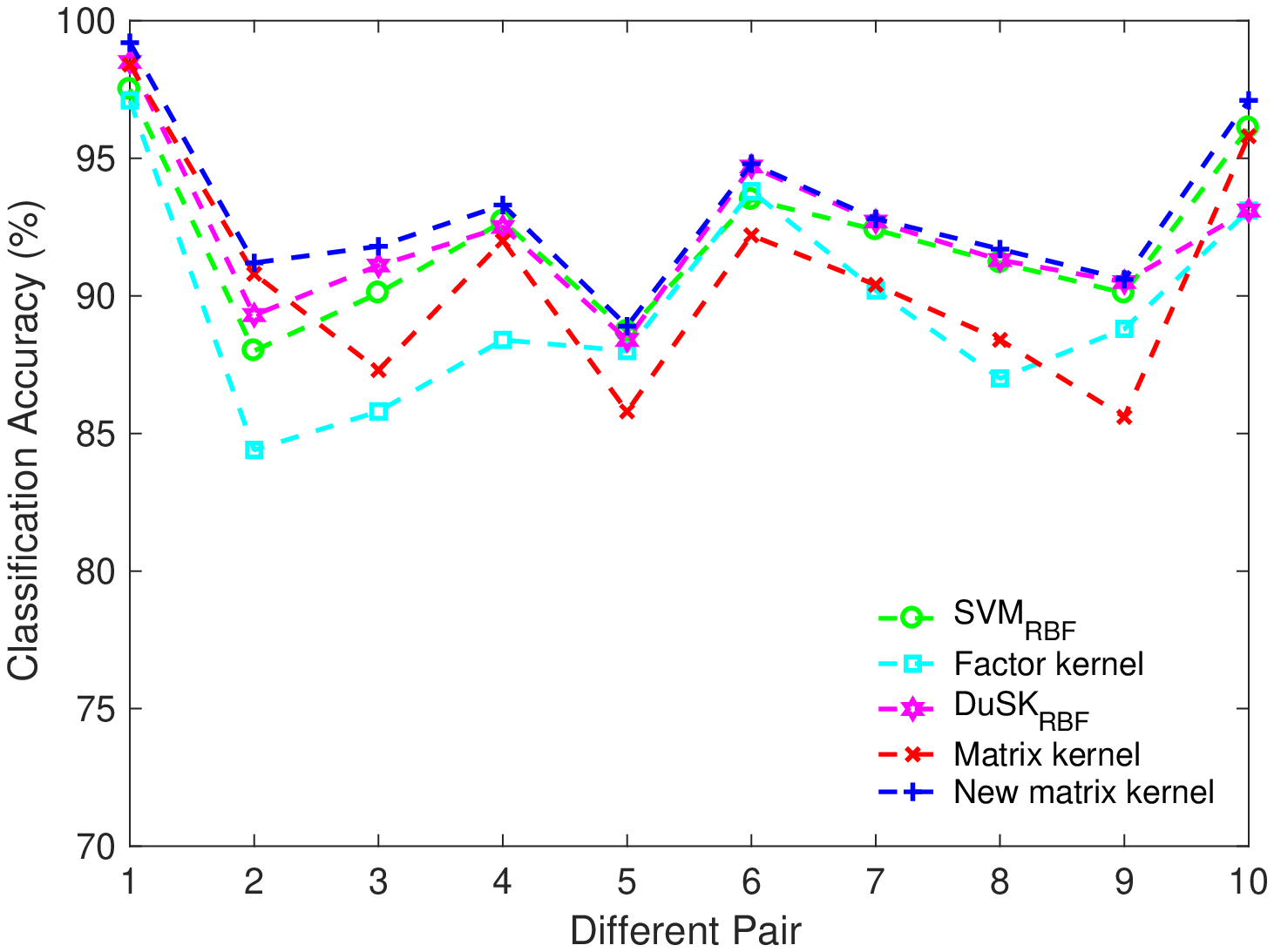}}
\subfigure[100 train]{
\includegraphics[width=0.48\textwidth]{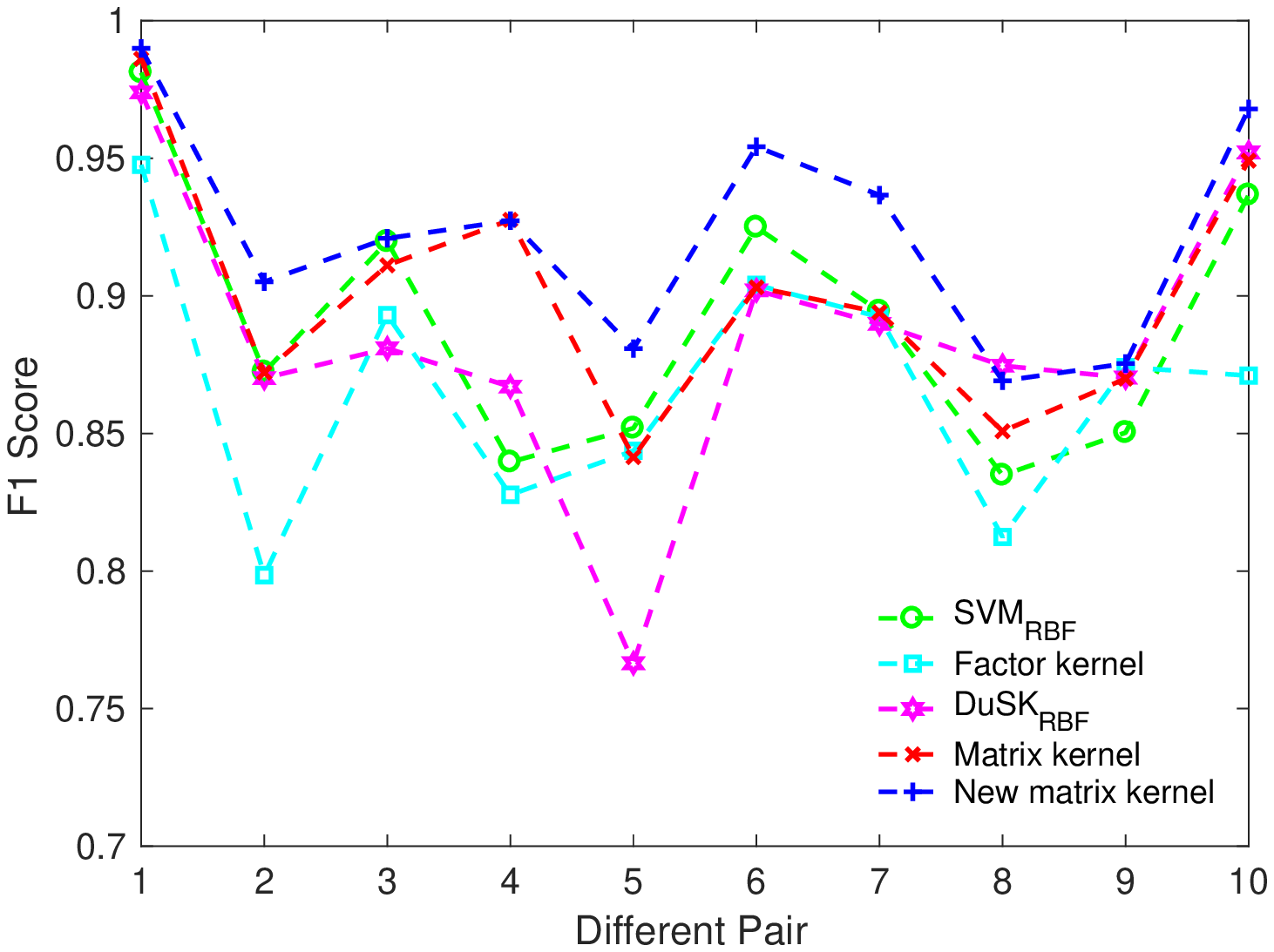}}
\subfigure[1000 train]{
\includegraphics[width=0.48\textwidth]{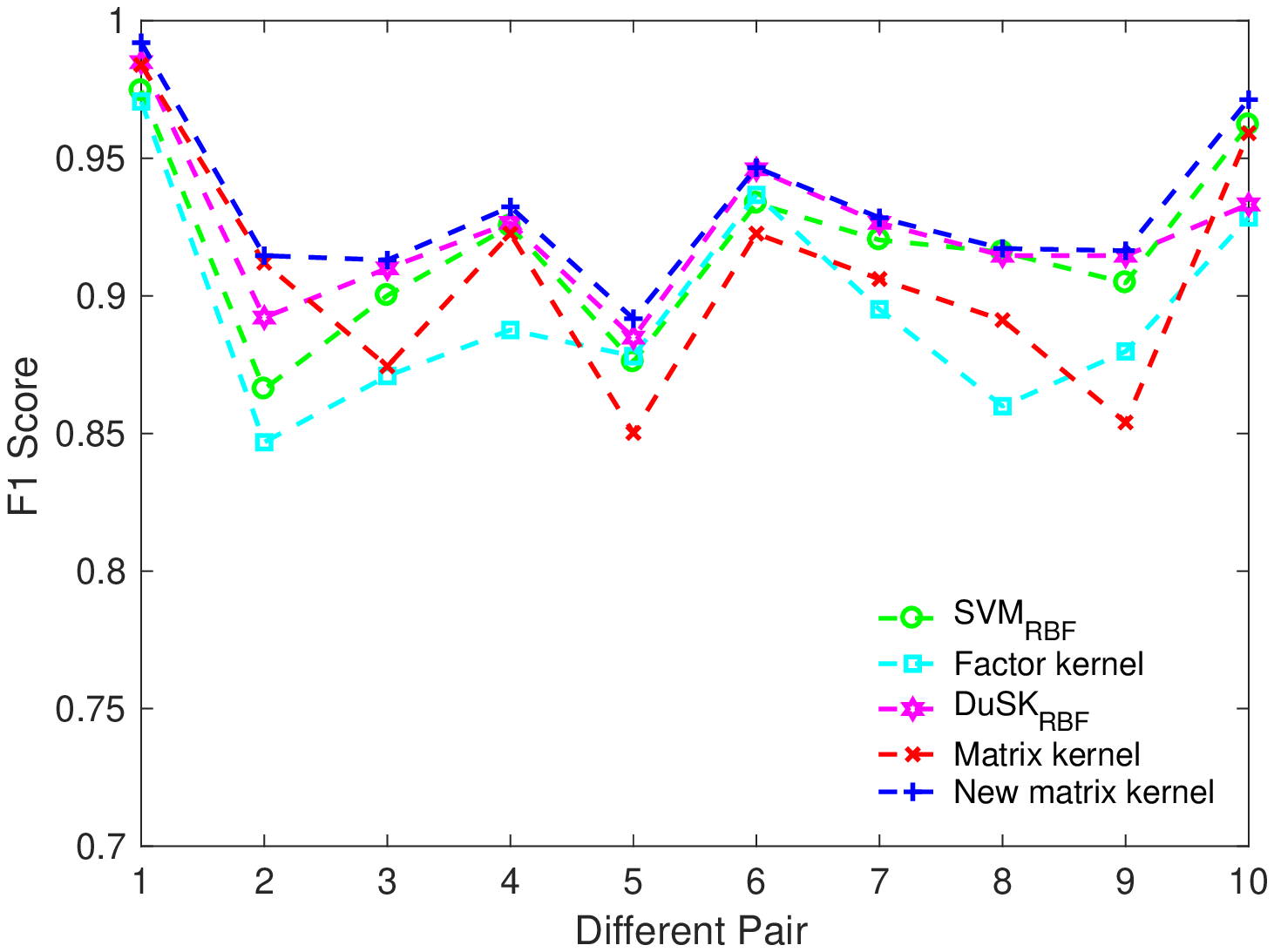}}
\caption{Average accuracy and $F_1$ score compared by different kernels for binary classification problem on MNIST dataset.}
\label{Fig2}
\end{figure}

Fig. \ref{Fig2} summarizes the results of $\rm SVM_{RBF}$, factor kernel, $\rm DuSK_{RBF}$, MRMLKSVM and new matrix kernel in terms of accuracy and $F_1$ score on MNIST dataset. The linear kernel-based SVM is not included because Gaussian-RBF kernel-based SVM proved to be better in the literature \citep{liu2003handwritten}. Similar patterns of curves are detected in accuracy and $F_1$ score among all the kernel methods. We can observe that our new matrix kernel performs well in general which demonstrates the effectiveness of the proposed algorithm. We are interested in accuracy and $F_1$ score in kernel comparison experiments and one way to understand this is to realize that matrices are calculated to construct our new kernel which occupies much more space and time. On the other hand, MRMLKSVM is competitive against $\rm SVM_{RBF}$, factor kernel and $\rm DuSK_{RBF}$, performs worse than our new matrix kernel. In addition, the observations demonstrate the size of training set has positive effect on the performance in most cases.

\begin{table}
\caption{Prediction performance of different kernels on experimental datasets in terms of accuracy and macro-averaged F-measure}
\begin{tabular*}{1\textwidth}{@{\extracolsep{\fill}}  lllll}
\hline
\noalign{\smallskip}
Kernel & \multicolumn{2}{l}{Accuracy $(\%)$} & \multicolumn{2}{l}{F-measure}\\
\noalign{\smallskip}
\cline{2-3} \cline{4-5}
\noalign{\smallskip}
& Yale Face & FingerDB & Yale Face & FingerDB \\
\noalign{\smallskip}
\hline
\noalign{\smallskip}
$\rm SVM_{linear}$ & 86.9(3.7) & 72.0(5.8) & 0.871(0.036) & 0.700(0.064) \\
$\rm SVM_{RBF}$ & 85.6(3.6) & 73.0(4.0) & 0.859(0.035) & 0.712(0.045) \\
Factor Kernel & 76.0(5.2) & 64.0(2.5) & 0.764(0.052) & 0.599(0.043) \\
$\rm DuSK_{RBF}$ & 80.8(3.2) & 73.5(2.5) & 0.808(0.033) & 0.732(0.033) \\
$\rm MRMLKSVM_{RBF}$  & 84.0(4.4) & 72.5(2.7) & 0.844(0.040) & 0.719(0.029)\\
$\rm New Matrix Kernel_{poly}$  & \textbf{89.3(4.9)} & \textbf{75.0(3.5)} & \textbf{0.892(0.049)} & \textbf{0.741(0.036)} \\
\noalign{\smallskip}
\hline
\label{Table1}
\end{tabular*}
\end{table}

Table \ref{Table1} reports the performance of several types of kernels with respect to accuracy and macro-averaged F-measure for multiple classification problems, where best results are highlighted in bold type. We can observe that both of accuracy and macro-averaged F-measure show the similar trend. Our new matrix kernel clearly benefits from its complexity and on both datasets it outperforms $\rm SVM_{linear}$, $\rm SVM_{RBF}$, $\rm DuSK_{RBF}$ and factor kernel in almost significant manner. Factor kernel performs significantly worse in all domains. On the other hand, MRMLKSVM performs slightly worse than our new matrix kernel and the performance of it is quite different on different data sets. This may due to the construction of the kernel where only the columns of matrix are under consideration, whereas information in the rows or inside the structure is lost.

Specifically, the learning algorithms which factorize the input
matrices into the product of vectors tend to have a lower performance
compared to those who preserve the initial structure.
This indicates that approximation by decomposition would lose the compact structural information within data, leading to the diminished performance.

So far we have compared all experimental results. The results of classification accuracy and F-measure for $\rm DuSK_{RBF}$, factor kernel, SVM, MRMLKSVM and our new matrix kernel demonstrate that our new matrix kernel is significantly effective on both binary and multiple classification problems. Notice that we apply polynomial kernel in the construction of the new matrix  kernel, and other types of vector kernel can also be adopted.

\section{Concluding Remarks}
\label{sec:5}

In this paper we have established a mathematical framework of matrix Hilbert spaces. Intuitively, we introduce a matrix inner product to generalize the scalar inner product, which is especially useful in the presence of structured data. The reproducing kernel and the reproducing kernel matrix Hilbert space in our framework allow one to construct various structure of kernels in nonlinear cases. In our experiments, kernel induced by our framework has favourable predictive performance on both binary and multiple classification problems.

Matrix Hilbert space provides several interesting avenues for future work. For calculating the new matrix kernel developing computational methods could improve efficiency. Since the dimension of matrix inner product has a negative effect on the usage of space and time, we could pick up an appropriate one to balance the trade-off between efficiency and accuracy. While experiments have been conducted on the classification problems, our results indicate that RKMHS is directly applicable to regression, clustering, and ranking, among other tasks.

\begin{acknowledgements}
The work is supported by National Natural Science Foundations of China under Grant 11531001 and National Program on Key Basic Research Project under Grant 2015CB856004. We are grateful to Dong Han and Lynn Kuo for our discussions.
\end{acknowledgements}

\bibliographystyle{spbasic} 
\bibliography{Ref}  

\begin{thebibliography}{29}
\providecommand{\natexlab}[1]{#1}
\providecommand{\url}[1]{{#1}}
\providecommand{\urlprefix}{URL }
\expandafter\ifx\csname urlstyle\endcsname\relax
  \providecommand{\doi}[1]{DOI~\discretionary{}{}{}#1}\else
  \providecommand{\doi}{DOI~\discretionary{}{}{}\begingroup
  \urlstyle{rm}\Url}\fi
\providecommand{\eprint}[2][]{\url{#2}}

\bibitem[{Aronszajn(1950)}]{aronszajn1950theory}
Aronszajn N (1950) Theory of reproducing kernels. Transactions of the American
  mathematical society 68(3):337--404

\bibitem[{Asgari and Khosravi(2003)}]{asgari2003frames}
Asgari M, Khosravi A (2003) Frames and bases in tensor product of hilbert
  spaces, intern. Math Journal 4(527-537)

\bibitem[{Ballentine(2014)}]{ballentine2014quantum}
Ballentine LE (2014) Quantum mechanics: a modern development. World Scientific
  Publishing Co Inc

\bibitem[{Belhumeur et~al(1997)Belhumeur, Hespanha, and
  Kriegman}]{belhumeur1997eigenfaces}
Belhumeur PN, Hespanha JP, Kriegman DJ (1997) Eigenfaces vs. fisherfaces:
  Recognition using class specific linear projection. Pattern Analysis and
  Machine Intelligence, IEEE Transactions on 19(7):711--720

\bibitem[{Birkhoff and Von~Neumann(1936)}]{birkhoff1936logic}
Birkhoff G, Von~Neumann J (1936) The logic of quantum mechanics. Annals of
  mathematics pp 823--843

\bibitem[{Birman and Solomjak(2012)}]{birman2012spectral}
Birman MS, Solomjak MZ (2012) Spectral theory of self-adjoint operators in
  Hilbert space, vol~5. Springer Science \& Business Media

\bibitem[{Boser et~al(1992)Boser, Guyon, and Vapnik}]{boser1992training}
Boser BE, Guyon IM, Vapnik VN (1992) A training algorithm for optimal margin
  classifiers. In: Proceedings of the fifth annual workshop on Computational
  learning theory, ACM, pp 144--152

\bibitem[{Crandall et~al(1992)Crandall, Ishii, and Lions}]{crandall1992user}
Crandall MG, Ishii H, Lions PL (1992) User¡¯s guide to viscosity solutions of
  second order partial differential equations. Bulletin of the American
  Mathematical Society 27(1):1--67

\bibitem[{Gao and Wu(2012)}]{Gao2012Kernel}
Gao C, Wu XJ (2012) Kernel support tensor regression. Procedia Engineering
  29(4):3986--3990

\bibitem[{Gao et~al(2015)Gao, Fan, and Xu}]{gao2015multiple}
Gao X, Fan L, Xu H (2015) Multiple rank multi-linear kernel support vector
  machine for matrix data classification. International Journal of Machine
  Learning and Cybernetics pp 1--11

\bibitem[{Gao et~al(2014)Gao, Fan, and Xu}]{Gao2014NLS}
Gao XZ, Fan L, Xu H (2014) Nls-tstm: A novel and fast nonlinear image
  classification method. Wseas Transactions on Mathematics 13:626--635

\bibitem[{Gilbarg and Trudinger(2015)}]{gilbarg2015elliptic}
Gilbarg D, Trudinger NS (2015) Elliptic partial differential equations of
  second order. springer

\bibitem[{Gustafson(2012)}]{gustafson2012introduction}
Gustafson KE (2012) Introduction to partial differential equations and Hilbert
  space methods. Courier Corporation

\bibitem[{Hao et~al(2013)Hao, He, Chen, and Yang}]{hao2013linear}
Hao Z, He L, Chen B, Yang X (2013) A linear support higher-order tensor machine
  for classification. IEEE Transactions on Image Processing 22(7):2911--2920

\bibitem[{He et~al(2014)He, Kong, Yu, Yang, Ragin, and Hao}]{he2014dusk}
He L, Kong X, Yu PS, Yang X, Ragin AB, Hao Z (2014) Dusk: A dual
  structure-preserving kernel for supervised tensor learning with applications
  to neuroimages. In: Proceedings of the 2014 SIAM International Conference on
  Data Mining, SIAM, pp 127--135

\bibitem[{Hofmann et~al(2008)Hofmann, Sch{\"o}lkopf, and
  Smola}]{hofmann2008kernel}
Hofmann T, Sch{\"o}lkopf B, Smola AJ (2008) Kernel methods in machine learning.
  The annals of statistics pp 1171--1220

\bibitem[{Horn(1990)}]{horn1990hadamard}
Horn RA (1990) The hadamard product. In: Proc. Symp. Appl. Math, vol~40, pp
  87--169

\bibitem[{Khosravi and Khosravi(2007)}]{khosravi2007frames}
Khosravi A, Khosravi B (2007) Frames and bases in tensor products of hilbert
  spaces and hilbert c*-modules. Proceedings Mathematical Sciences 117(1):1--12

\bibitem[{Lance(1995)}]{lance1995hilbert}
Lance EC (1995) Hilbert C*-modules: a toolkit for operator algebraists, vol
  210. Cambridge University Press

\bibitem[{LeCun et~al(1998)LeCun, Bottou, Bengio, and
  Haffner}]{lecun1998gradient}
LeCun Y, Bottou L, Bengio Y, Haffner P (1998) Gradient-based learning applied
  to document recognition. Proceedings of the IEEE 86(11):2278--2324

\bibitem[{Liu et~al(2003)Liu, Nakashima, Sako, and
  Fujisawa}]{liu2003handwritten}
Liu CL, Nakashima K, Sako H, Fujisawa H (2003) Handwritten digit recognition:
  benchmarking of state-of-the-art techniques. Pattern recognition
  36(10):2271--2285

\bibitem[{Sakurai et~al(1995)Sakurai, Tuan, and Commins}]{sakurai1995modern}
Sakurai JJ, Tuan SF, Commins ED (1995) Modern quantum mechanics, revised
  edition

\bibitem[{Sch{\"o}lkopf et~al(1998)Sch{\"o}lkopf, Smola, and
  M{\"u}ller}]{scholkopf1998nonlinear}
Sch{\"o}lkopf B, Smola A, M{\"u}ller KR (1998) Nonlinear component analysis as
  a kernel eigenvalue problem. Neural computation 10(5):1299--1319

\bibitem[{Signoretto et~al(2011)Signoretto, De~Lathauwer, and
  Suykens}]{signoretto2011kernel}
Signoretto M, De~Lathauwer L, Suykens JA (2011) A kernel-based framework to
  tensorial data analysis. Neural networks 24(8):861--874

\bibitem[{Sinap and Van~Assche(1994)}]{sinap1994polynomial}
Sinap A, Van~Assche W (1994) Polynomial interpolation and gaussian quadrature
  for matrix-valued functions. Linear algebra and its applications 207:71--114

\bibitem[{Stein and Weiss(2016)}]{stein2016introduction}
Stein EM, Weiss G (2016) Introduction to Fourier analysis on Euclidean spaces
  (PMS-32), vol~32. Princeton university press

\bibitem[{Vapnik and Vapnik(1998)}]{vapnik1998statistical}
Vapnik VN, Vapnik V (1998) Statistical learning theory, vol~1. Wiley New York

\bibitem[{Wahba(1990)}]{wahba1990spline}
Wahba G (1990) Spline models for observational data. SIAM

\bibitem[{Yang and Liu(1999)}]{yang1999re}
Yang Y, Liu X (1999) A re-examination of text categorization methods. In:
  Proceedings of the 22nd annual international ACM SIGIR conference on Research
  and development in information retrieval, ACM, pp 42--49

\end{thebibliography}

\end{document}